\documentclass[final,12pt,nosubfloats,backref=page]{clear2023} %

\title[Backtracking Counterfactuals]{Backtracking Counterfactuals}
\usepackage{times}
\usepackage{enumitem}

\usepackage{bm}
\usepackage{notation}
\usepackage{subcaption}
\usepackage{tikz}
\usetikzlibrary{bayesnet}
\usepackage{cleveref}
\crefname{section}{\S}{\S\S}
\crefname{appendix}{Appendix}{Appendices}
\crefname{subsection}{\S}{\S\S}
\crefname{subsubsection}{\S}{\S\S}
\crefname{figure}{Figure}{Figures}
\crefname{table}{Table}{Tables}
\crefname{definition}{Definition}{Definitions}
\crefname{corollary}{Corollary}{Corollaries}
\crefname{proposition}{Proposition}{Propositions}
\crefname{theorem}{Theorem}{Theorems}
\crefname{remark}{Remark}{Remarks}
\crefname{principle}{Principle}{Principles}
\crefname{example}{Example}{Examples}
\crefname{lemma}{Lemma}{Lemmata}
\crefname{claim}{Claim}{Claims}
\numberwithin{equation}{section}
\usepackage{quoting}
\usepackage{lipsum}
\renewcommand*{\backref}[1]{}
\renewcommand*{\backrefalt}[4]{%
    \ifcase #1%
          \or [Cited on page~#2.]%
          \else [Cited on pages~#2.]%
    \fi%
    }
\clearauthor{%
 \Name{Julius {von K\"ugelgen}} \Email{jvk@tuebingen.mpg.de}\\
 \addr Max Planck Institute for Intelligent Systems T\"ubingen, Germany \\ University of Cambridge
 \AND
 \Name{Abdirisak Mohamed} \Email{amoham70@umd.edu}\\
 \addr University of Maryland\\ SAP Labs, LLC
 \AND
 \Name{Sander Beckers} \Email{srekcebrednas@gmail.com }\\
 \addr University of Amsterdam
}

\begin{document}
\everypar{\looseness=-1}
\maketitle
\begin{abstract}%
Counterfactual reasoning---envisioning hypothetical scenarios, or possible worlds, where some circumstances are different from what (f)actually occurred (counter-to-fact)---is ubiquitous in human cognition.
Conventionally, counterfactually-altered circumstances have been treated as ``small miracles''
that locally violate the laws of nature while sharing the same initial conditions.
In Pearl's
structural causal model (SCM) framework this is made mathematically rigorous via \textit{interventions} that modify the causal laws while the values of exogenous variables are shared.
In recent years, however, this purely interventionist account of counterfactuals has increasingly come under scrutiny from both philosophers and psychologists.
Instead, they suggest a {\em backtracking} account of counterfactuals, according to which \textit{the causal laws remain unchanged} in the counterfactual world; differences to the factual world are instead ``backtracked'' to altered initial conditions (exogenous variables).
In the present work, we explore and formalise this alternative mode of counterfactual reasoning within the SCM framework.
Despite ample evidence that humans backtrack, the present work constitutes, to the best of our knowledge, the first general account and algorithmisation of backtracking counterfactuals.
We discuss our backtracking semantics in the context of related literature and draw connections to recent developments in explainable artificial intelligence (XAI).
\end{abstract}
\section{Introduction}
\label{sec:intro}
In a deterministic world, everything that happens is uniquely determined by the laws of nature and the initial (or background) conditions.\footnote{Even if the world is fundamentally non-deterministic at the quantum level, the point still holds for deterministic {\em models} of the world.
Causal models in particular can be given a deterministic interpretation, see~\eqref{eq:endodistr}.} Counterfactuals invite us to imagine what the world would look like if some events which did occur, had in fact not occurred. As a result, in order to make sense of counterfactuals in a deterministic world one is immediately presented with the following dilemma: for events to have been different, either the laws of nature would have had to be violated, or the background conditions would have had to be different. 

Following in the footsteps of~\citet{lewis1973counterfactuals,lewis1979counterfactual}, the majority of philosophers have opted for the first option: counterfactuals are to be evaluated by imagining ``small miracles'' that ensure those events which are counter-to-fact to occur by locally violating the laws of nature, thereby disconnecting these events from their causes, and keeping the background conditions fixed. In other words, counterfactuals are {\em non-backtracking}. \Citet[][p.239]{pearl2009causality} objected to Lewis's miracles and the possible-world semantics that comes with it, instead replacing these with structural equations and the well-behaved notion of an intervention that those allow:
\begin{quoting}
    Lewis’s elusive ``miracles'' are replaced by principled minisurgeries, [...] which represent the minimal change (to a model) necessary for establishing the antecedent.
\end{quoting}%
Still, Pearl's \textit{interventional} counterfactuals are also a commitment to the first option of our dilemma.

Recently this status quo has come under pressure by several philosophers who argue in favour of the second option~\citep{dorr2016against,Loewer2007-LOECAT,loewer2020consequence,esfeld2021super}: counter-to-fact events are to be explained by imagining small changes to the background conditions that would result in the counterfactual events instead of the actual ones, while the laws of nature remain unchanged. In other words, counterfactuals are {\em backtracking}.\footnote{or {\em observational}, we use both terms interchangeably; likewise for non-backtracking and interventional} Others propose a combination of both options, suggesting that each is appropriate under different circumstances \citep{fisher2017causal,fisher2017counterlegal,woodward2021causation}.   

This dispute is not merely academic: empirical work by psychologists now confirms that, depending on the context, people indeed switch between interpreting counterfactuals according to the first and the second option~\citep{rips2010two,gerstenberg2013back,lucas2015improved}. Historical explanations often rely on counterfactuals as well, and it has been argued that ``counterfactuals in history are backtracking'' \cite[p. 719]{reiss2009counterfactuals}.
Moreover, counterfactual explanations are taking up a prominent role in the work on explainable AI;  their interpretation thus has significant repercussions for real-life applications~\citep{wachter2017counterfactual}.

\begin{figure}[t]
    \centering
    \newcommand{\xshift}{5em}
    \newcommand{\yshift}{5em}
    \begin{subfigure}[b]{0.5\textwidth}
        \centering
        \begin{tikzpicture}
            \centering
            \node(u)[]{$\ub$};
            \node(V)[yshift=-\yshift,xshift=-\xshift]{$\Vb$};
            \node(V*)[yshift=-\yshift,xshift=\xshift]{$\Vb^*$};
            \edge {u}{V};
            \node[const,yshift=-0.35*\yshift,xshift=-0.5*\xshift]{$\Fb$};
            \edge{u}{V*};
            \node[const,yshift=-0.35*\yshift,xshift=0.6*\xshift]{$\Fb^*$};
            \node(wI)[const,yshift=-1.5*\yshift,xshift=-\xshift]{$w_I=(\Mcal,\ub)$};
            \node(wI*)[const,yshift=-1.5*\yshift,xshift=\xshift]{$w_I^*=(\Mcal^*,\ub)$};
        \end{tikzpicture}
        \caption{Interventional counterfactual}
        \label{subfig:IC}
    \end{subfigure}%
    \begin{subfigure}[b]{0.5\textwidth}
        \centering
        \begin{tikzpicture}
            \centering
            \node(u)[]{$\ub$};
            \node(V)[yshift=-\yshift]{$\Vb$};
            \node(u*)[xshift=2*\xshift]{$\ub^*$};
            \node(V*)[yshift=-\yshift,xshift=2*\xshift]{$\Vb^*$};
            \edge {u}{V};
            \node[const,yshift=-0.35*\yshift,xshift=-0.1*\xshift]{$\Fb$};
            \edge{u*}{V*};
            \path[<->,dashed] (u) edge (u*);
            \node[const,yshift=-0.35*\yshift,xshift=1.9*\xshift]{$\Fb$};
            \node(wO)[const,yshift=-1.5*\yshift]{$w_B=(\Mcal,\ub)$};
            \node(wO*)[const,yshift=-1.5*\yshift,xshift=2*\xshift]{$w_B^*=(\Mcal,\ub^*)$};
        \end{tikzpicture}
        \caption{Backtracking counterfactual}
        \label{subfig:BC}
    \end{subfigure}
    \caption{\looseness-1 %
    \textbf{Illustration of the main conceptual difference between interventional and backtracking counterfactuals.} 
    (a)~For \textit{interventional} counterfactuals, the factual world~$w_I$ and counterfactual world~$w_I^*$ share the exact same background conditions~$\ub$. Potential contradictions between the factual outcome~$\Vb$ and the counterfactual outcome~$\Vb^*$ are resolved through changes to the causal laws~$\Fb$ (by means of intervention), giving rise to the modified laws~$\Fb^*$ and submodel~$\Mcal^*$.
    (b)~For \textit{backtracking} counterfactuals, on the other hand, the factual world~$w_B$ and counterfactual world~$w_B^*$ share the same unmodified causal laws~$\Fb$. To allow for the factual and counterfactual outcomes~$\Vb$ and~$\Vb^*$ to differ, the respective background conditions $\ub$ and $\ub^*$ may differ, too, with a preference for minimal changes, that is, for the two worlds to be as close as possible.
    }
    \label{fig:overview}
\end{figure}
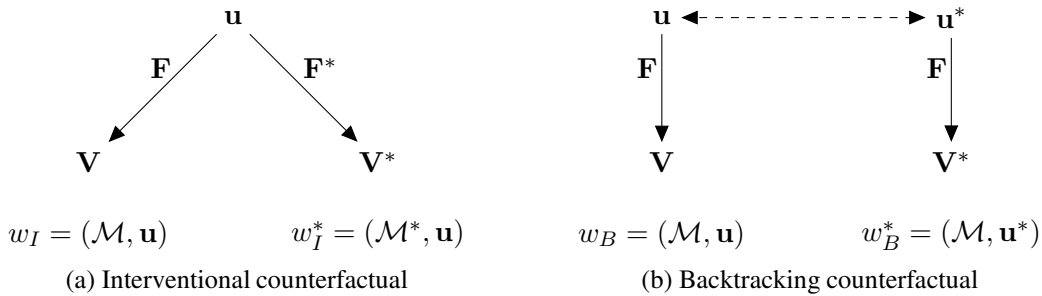

While non-backtracking, interventional counterfactuals have been given a well-defined semantics within \citeauthor{pearl2009causality}'s framework of structural causal models~(SCMs), the same cannot be said for backtracking counterfactuals. 
In this paper we offer the first general formal semantics for backtracking counterfactuals within SCMs. 
Our general semantics can be made more specific in various ways, depending on the particular purpose that is at stake. We do not view these backtracking semantics as an alternative to the standard non-backtracking semantics: the value of the latter for causal reasoning is undeniable. Rather, we see both semantics as being fit for different reasoning tasks. 
\begin{quoting}%
it is appropriate to use backtracking counterfactuals to answer 
[...] how the past would have had to have been different had the present been different. [...]
backtracking counterfactuals are important in diagnostic reasoning. However, this does not mean that it is misguided to use non-backtracking counterfactuals to answer other sorts of questions such as those having to do with whether Cs cause Es. The two kinds of counterfactuals are just different, with different truth conditions~\citep[p.~206]{woodward2021causation}
\end{quoting}

\paragraph{Structure and Contributions.}
We start by reviewing the SCM framework and its interventional semantics for counterfactual reasoning~(\cref{sec:background}).
In~\cref{sec:backtracking}, we then provide a comprehensive account of the alternative, backtracking mode of counterfactual reasoning by explaining the high-level intuition~(\cref{subsec:intuition}, see~\cref{fig:overview} for an overview), introducing a formal semantics~(\cref{subsec:formalisation}), discussing some of the design choices and desiderata~(\cref{subsec:choices}), and providing some first theoretical insights~(\cref{subsec:theory}).
We compare our semantics to previous attempts and discuss related work from other fields in~\cref{sec:related_work}. We then highlight the use of backtracking counterfactuals for explainable artificial intelligence (XAI)~(\cref{sec:xai}), and conclude with an outlook and suggestions for future work~(\cref{sec:future_work}).

\section{Preliminaries: Structural Causal Models and Interventional Counterfactuals}
\label{sec:background}
The following definitions of SCMs and their counterfactual semantics follow~\citet[][\S~7.1]{pearl2009causality}.

A \textbf{causal model} is a triple $\mathcal{M}=(\mathbf{U},\mathbf{V}, \mathbf{F})$
where:
    (i) 
    $\mathbf{U}$ is a set $\{U_1, ..., U_m\}$ of \textbf{exogenous} (background) variables 
    determined by factors outside the model;
    (ii) 
    $\mathbf{V}$ is a set $\{V_1, V_2,\dots, V_n\}$ of \textbf{endogenous} variables 
    determined by variables in the model, that is, variables in $\mathbf{U}\cup\mathbf{V}$;
    and 
    (iii)~$\mathbf{F}$ is a set of functions $\{f_1, f_2,\dots, f_n\}$ s.t.\ each $f_i$ is a mapping from (the
    respective domains of) $\mathbf{U}_i\cup \mathbf{PA}_i$ to~$V_i$, where $\mathbf{U}_i\subseteq \mathbf{U}$ and $\mathbf{PA}_i\subseteq \mathbf{V}\setminus \{V_i\}$, and the
    entire set $\mathbf{F}$ forms a mapping from $\mathbf{U}$ to $\mathbf{V}$. In other words, each $f_i$ in the \textbf{structural equations}, or \textbf{causal laws},
    \begin{equation}
        V_i:=f_i(\mathbf{PA}_i, \mathbf{U}_i) \qquad \qquad i=1,\dots, n,
    \end{equation}
    assigns a value to $V_i$ that depends on (the values of) a select set of variables in $\mathbf{U}\cup\mathbf{V}$,
    and the entire set $\mathbf{F}$ has a unique solution $\mathbf{V}(\mathbf{u})$. The latter is ensured, e.g., in acyclic (``recursive'') systems.

The \textbf{causal diagram} $G(\mathcal{M})$ associated with causal model $\mathcal{M}$ is the directed graph in which each node corresponds to a variable and directed edges point from members of $\mathbf{PA}_i$ and
$\mathbf{U}_i$ toward~$V_i$. 
Since the exogenous (background) variables $\Ub$ are typically unobserved, it is common to only consider the subset of $G(\mathcal{M})$ corresponding to its projection onto $\Vb$, where shared exogenous parents between some $V_i$ and $V_j$ are indicated with a bi-directed arrow $V_i\leftrightarrow V_j$

For a subset of endogenous variables $\mathbf{X}\subseteq\mathbf{V}$ and a realisation $\mathbf{x}$ thereof, the \textbf{submodel} $\mathcal{M}_\mathbf{x}$ of~$\mathcal{M}$ is the model $\mathcal{M}_\mathbf{x}=(\mathbf{U},\mathbf{V}, \mathbf{F}_\mathbf{x})$ where 
$\mathbf{F}_\mathbf{x}=\{f_i:V_i\not\in\mathbf{X}\}\cup \{\mathbf{X}:=\mathbf{x}\}$.
The \textbf{effect of action} $do(\mathbf{X} = \mathbf{x})$ on $\mathcal{M}$ is given by the submodel~$\mathcal{M}_\mathbf{x}$.
A \textbf{causal world} $w$ is a pair $(\Mcal,\ub)$ where $\Mcal$ is a causal model and $\ub$ is a particular realization
of the background variables $\Ub$.
The \textbf{potential response} of $\mathbf{Y}\subseteq \mathbf{V}$ to action $do(\mathbf{X} = \mathbf{x})$ in world $w=(\Mcal,\ub)$, denoted $\mathbf{Y}_{\mathbf{x}}(\mathbf{u})$, is the solution for $\mathbf{Y}$ of the set of equations $\mathbf{F}_\mathbf{x}$, that is, $\mathbf{Y}_\mathbf{x}(\mathbf{u})=\mathbf{Y}_{\mathcal{M}_\mathbf{x}}(\mathbf{u})$.
The \textbf{counterfactual} sentence ``$\mathbf{Y}$ would be $\mathbf{y}$ (in situation $\mathbf{u}$), had $\mathbf{X}$ been $\mathbf{x}$'' is then interpreted as the equality $\mathbf{Y}_{\mathbf{x}}(\mathbf{u})=\mathbf{y}$.
The part ``had $\mathbf{X}$ been $\mathbf{x}$'' is called the (counterfactual) \textbf{antecedent}.

A \textbf{probabilistic causal model} is a distribution over causal worlds, that is, a pair $(\mathcal{M},P(\mathbf{U}))$ where $\Mcal$ is a causal model and $P(\Ub)$ is a probability function defined over the domain of $\mathbf{U}$.
The function $P(\Ub)$, together with the fact that each endogenous variable is a function of $\Ub$, defines a \textbf{distribution over endogenous variables}: for any $\Yb\subseteq \Vb$ we have\footnote{To focus on the main points, ease the notation, and avoid measure theoretic details, we present all formulae for \textit{discrete} random variables in terms of \textit{sums} of probability \textit{mass} functions; the analogue for \textit{continuous} random variables would involve \textit{integrals} of probability density functions (assuming existence of densities w.r.t.\ the Lebesgue measure).}
\begin{equation}\label{eq:endodistr}
    P(\yb):=P(\Yb=\yb)=\sum_{\ub} P(\ub) \, \1_{\{\Yb(\ub)=\yb\}}
\end{equation}
where $\1$ denotes the indicator function.
The \textbf{probability of counterfactuals}
is defined analogously, through the potential responses induced by different submodels:
for any not necessarily disjoint sets of variables $\Yb, \Xb, \Zb, \Wb\subseteq \Vb$, we have
\begin{equation}
\label{eq:joint_prob_ICs}
    P(\Yb_\xb = \yb, \Zb_\wb = \zb)=\sum_{\ub} P(\ub)\1_{\{\Yb_\xb(\ub)=\yb\}}\1_{\{\Zb_\wb(\ub) = \zb\}}.
\end{equation}

In particular, $P(\Yb_\xb=\yb, \Xb=\xb')$ and $P(\Yb_\xb=\yb, \Yb_{\xb'}=\yb')$ are well defined using SCM semantics, even though $\xb\neq\xb'$ may be incompatible and thus cannot be measured simultaneously.\footnote{Some therefore consider statements about such expressions as fundamentally unscientific, see, e.g.,~\citet{dawid2000causal}.}

To notationally distinguish the factual variables~$\Vb$ from counterfactual versions, or copies, thereof, the latter are sometimes also denoted with an asterisk~$\Vb^*$~\citep[][]{BalkeP94}.

Of special interest are counterfactuals of the form $P(\Yb^*_{\xb^*}=\yb^*~|~\zb)$ that are conditional on a (f)actual observation $\zb$. 
These can be computed via the following three-step procedure:
\begin{enumerate}[left=0pt .. \parindent]
    \item \textbf{Abduction:} Update $P(\Ub)$ by the evidence $\zb$ to obtain $P(\Ub~|~\zb)$.
    \item \textbf{Action:} Modify $\Mcal$ by the action $do(\xb^*)$
    to obtain the submodel $\Mcal_{\xb^*}$.
    \item \textbf{Prediction:} Use the modified model $(\Mcal_{\xb^*}, P(\Ub~|~\zb))$ to compute the probability of $\Yb^*$.
\end{enumerate}

\begin{example}[Interventional counterfactual]
\label{ex:IC}
Consider a probabilistic causal model $(\Mcal,P(\Ub))$ with 
$\Vb=\{X,Y,Z\}$, $\Ub=\{U_X,U_Y, U_Z\}$, causal laws $\Fb$ given by
\begin{equation}
\label{eq:example_linear_SCM}
    X:=U_X, \qquad\qquad 
    Y:=X+U_Y, \qquad\qquad
    Z:=X+Y+U_Z,
\end{equation}
and $P(\Ub)$ being the multivariate standard isotropic Gaussian distribution $\Ncal(\zerob,\Ib_3)$.

Suppose that we make the factual observation $(X=1, Y=2, Z=2)$ and wish to reason about the (interventional) counterfactual ``what would have been, had $Y$ been $3$'', that is, we are interested in $P((X^*,Z^*)_{Y^*=3}~|~X=1,Y=2,Z=2)$, see~\cref{subfig:ex_IC} for a visualisation.

To this end, we follow the aforementioned three-step procedure:
\begin{enumerate}[left=0pt .. \parindent]
    \item \textbf{Abduction:}
    we find the posterior~$P(\Ub~|~X=1, Y=2, Z=2)$ to be a point mass on $(U_X=1,U_Y=1, U_Z=-1)$.
    \item \textbf{Action:} 
    we modify $\Mcal$ to obtain the submodel $\Mcal_{Y^*=3}$ as
\begin{equation}
\label{eq:ex_IC}
    X^*:=U_X, \qquad\qquad
    Y^*:=3, \qquad\qquad
    Z^*:=X^*+Y^*+U_Z;
\end{equation}
Note that this step alters the mechanism for the counterfactual antecedent $Y^*$ and removes its dependence on $X^*$ and $U_Y$, as shown in~\cref{subfig:ex_IC}.
\item \textbf{Prediction:}
we compute the push-forward of $P(\Ub~|~X=1, Y=2, Z=2)$ via~\eqref{eq:ex_IC} which yields a point mass on $(X^*=1, Z^*=3)$. That is, $X$ (being a parent of $Y$) would have remained unaffected, but $Z$ (being a child of $Y$) would have increased by~one.
%
\end{enumerate}
\end{example}
\iftrue
\begin{remark}[Ladder of Causation]
The distribution $P((X^*,Z^*)_{Y^*=3}~|~X=1,Y=2,Z=2)$ of interest in~\cref{ex:IC} differs from both simple conditioning, that is, the observational distribution 
\begin{equation}
    P(X,Z~|~Y=3)
    =\Ncal\left(
    \begin{pmatrix}
    1.5 \\ 4.5 
    \end{pmatrix},
    \begin{pmatrix}
    0.5 & 0.5 \\
    0.5 & 0.5 
    \end{pmatrix}
    \right),
\end{equation}
as well as from the interventional distribution 
\begin{equation}
    P((X,Z)_{Y=3})
    =\Ncal\left(
    \begin{pmatrix}
    0 \\ 3 
    \end{pmatrix}
    ,
    \begin{pmatrix}
    1 & 1 \\
    1 & 2 
    \end{pmatrix}
    \right),
\end{equation}
both of which preserve uncertainty about $(X,Z)$.
(In contrast, the interventional counterfactual is fully determined in~\cref{ex:IC}, because the shared background conditions $\ub$ could be uniquely inferred from the factual observation.)
These three modes of reasoning---observational, interventional, and counterfactual---constitute increasingly difficult tasks, each requiring additional data or assumptions over the previous, and form the three rungs of the so-called ``Ladder of Causation''~\citep{pearl2018book} or ``Pearl Causal Hierarchy''~\citep{bareinboim2022pearl}.
\end{remark}
\fi
\begin{remark}[Forward-tracking]
\label{remark:forward_tracking}
In acyclic causal models, any manipulation $\Fb_\xb$ of the causal laws~$\Fb$ such as $Y^*:=3$ in~\eqref{eq:ex_IC}
only has ``downstream'' effects: the change only propagates to descendants of the intervened-upon variables, while any non-descendants remain unaffected.
In this sense, interventions and thus also interventional counterfactuals are purely {\em forward-tracking}.
\end{remark}

The \textit{interventional} approach to computing counterfactuals in SCMs thus assumes that the background conditions $\ub$ are shared between the factual and counterfactual worlds and instead relies on modifying the causal laws (the \textit{action} step) to explain possible discrepancies (such as the difference $Y=2\neq Y^*=3$
in~\cref{ex:IC}),
as illustrated in~\cref{subfig:IC,subfig:ex_IC}.
As~\citet[][p.205]{pearl2009causality} puts it:~%
\begin{quoting}[]
[It]
interprets the counterfactual phrase ``had $\mathbf{X}$ been $\mathbf{x}$'' in terms of a {\em hypothetical modification of the equations in the model}; it simulates an {\em external action} (or spontaneous change) that modifies the actual course of history and enforces the condition ``$\mathbf{X}=\mathbf{x}$'' with minimal {\em change of mechanisms}. 
This 
[...]
permits $\mathbf{x}$ to differ from the current value of $\mathbf{X}(\mathbf{u})$ without creating logical contradiction; 
it also suppresses abductive inferences (or {\em backtracking}) from the counterfactual antecedent $\mathbf{X}=\mathbf{x}$~[{\em emphasis ours}]
\end{quoting}

\section{Backtracking Counterfactuals}
\label{sec:backtracking}
We now explore and formalise an alternative, non-interventional mode of counterfactual reasoning that does not involve such ``change of mechanism'' by ``external action'', and instead relies on  so-called \textit{backtracking} all changes to changes in the values of exogenous variables.

\begin{figure}[t]
    \centering
    \newcommand{\xshift}{5em}
    \newcommand{\yshift}{3.5em}
    \begin{subfigure}[b]{0.5\textwidth}
        \centering
        \begin{tikzpicture}
            \centering
            \node(X)[obs]{$X$};
            \node(Y)[obs,yshift=-\yshift]{$Y$};
            \node(Z)[obs,yshift=-2*\yshift]{$Z$};
            \node(X*)[latent,xshift=2*\xshift]{$X^*$};
            \node(Y*)[det,fill=gray!25,yshift=-\yshift,xshift=2*\xshift]{$Y^*$};
            \node(Z*)[latent,yshift=-2*\yshift,xshift=2*\xshift]{$Z^*$};
            \node(U_X)[latent,yshift=0.5*\yshift,xshift=\xshift]{$U_X$};
            \node(U_Y)[latent,yshift=-0.5*\yshift,xshift=\xshift]{$U_Y$};
            \node(U_Z)[latent,yshift=-1.5*\yshift,xshift=\xshift]{$U_Z$};
            \edge{U_X}{X,X*};
            \edge{X,U_Y}{Y};
            \edge{Y,U_Z}{Z};
            \edge{U_Z,Y*}{Z*};
            \path[->] (X) edge[bend right=45] (Z);
            \path[->] (X*) edge[bend left=45] (Z*);
        \end{tikzpicture}
        \caption{Interventional counterfactual~(\cref{ex:IC})}
        \label{subfig:ex_IC}
    \end{subfigure}%
    \begin{subfigure}[b]{0.5\textwidth}
        \centering
        \begin{tikzpicture}
            \centering
            \node(X)[obs]{$X$};
            \node(Y)[obs,yshift=-\yshift]{$Y$};
            \node(Z)[obs,yshift=-2*\yshift]{$Z$};
            \node(U_X)[latent,yshift=0.5*\yshift,xshift=\xshift]{$U_X$};
            \node(U_Y)[latent,yshift=-0.5*\yshift,xshift=\xshift]{$U_Y$};
            \node(U_Z)[latent,yshift=-1.5*\yshift,xshift=\xshift]{$U_Z$};
            \node(U_X*)[latent,yshift=0.5*\yshift,xshift=2*\xshift]{$U_X^*$};
            \node(U_Y*)[latent,yshift=-0.5*\yshift,xshift=2*\xshift]{$U_Y^*$};
            \node(U_Z*)[latent,yshift=-1.5*\yshift,xshift=2*\xshift]{$U_Z^*$};
            \node(X*)[latent,xshift=3*\xshift]{$X^*$};
            \node(Y*)[obs,yshift=-\yshift,xshift=3*\xshift]{$Y^*$};
            \node(Z*)[latent,yshift=-2*\yshift,xshift=3*\xshift]{$Z^*$};
            \edge{U_X}{X};
            \edge{U_X*}{X*};
            \edge{X,U_Y}{Y};
            \edge{X*,U_Y*}{Y*};
            \edge{Y,U_Z}{Z};
            \edge{Y*,U_Z*}{Z*};
            \path[->] (X) edge[bend right=45] (Z);
            \path[->] (X*) edge[bend left=45] (Z*);
            \path[<->,dashed] (U_X) edge (U_X*);
            \path[<->,dashed] (U_Y) edge (U_Y*);
            \path[<->,dashed] (U_Z) edge (U_Z*);
        \end{tikzpicture}
        \caption{Backtracking counterfactual~(\cref{ex:BC})}
        \label{subfig:ex_BC}
    \end{subfigure}
    \caption{%
    \textbf{Graphical comparison between the interventional and backtracking approaches to computing the counterfactual $P(X^*,Z^*~|~Y^*,X,Y,Z)$.}
    Shaded nodes are ``observed'' (or hypothesised to be known, in the case of counterfactual/starred variables), while white nodes are unobserved/latent. 
    (a)~For \textit{interventional} counterfactuals, the counterfactual antecedent~$Y^*$ is assumed to arise from a small miracle in the form of a (local) change to the causal laws, illustrated by the diamond node. This intervention removes the dependence of $Y^*$ on its causal parents $X^*$ and $U_Y$, hence the missing arrows into $Y^*$. The exogenous variables are assumed to be shared, giving rise to a so-called ``twin network''~\citep{BalkeP94}. As a consequence, non-descendants of~$Y^*$ (here, only the parent $X^*$), remain the same and only descendants of~$Y^*$ (here, only the child $Z^*$) are affected.
    (b)~For \textit{backtracking} counterfactuals, the causal relationship between the counterfactual antecedent~$Y^*$ and its parents $X^*$ and $U_X^*$ is preserved. This necessitates the introduction of new counterfactual exogenous variables $(U_X^*,U_Y^*,U_Z^*)$ which should be ``close'' to the original $(U_X,U_Y,U_Z)$ but may differ from it to absorb, or explain, the difference between $Y$ and $Y^*$. As a result of such backtracking, not only descendants ($Z^*$) but also ancestors ($X^*$) may change. 
    }
    \label{fig:examples}
\end{figure}
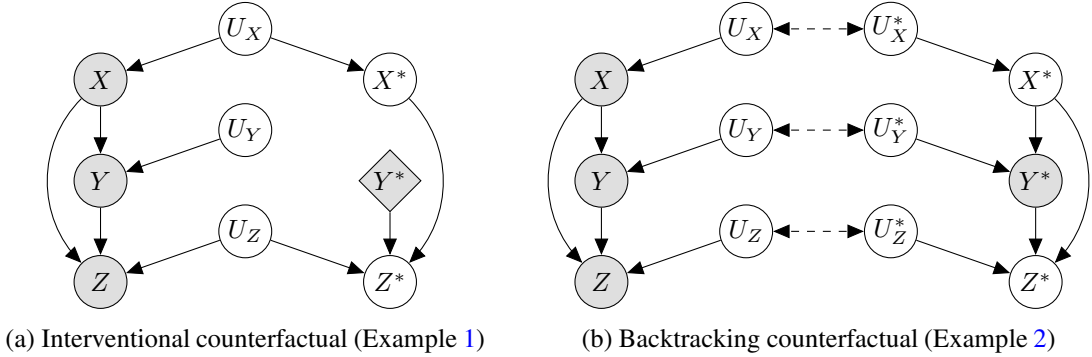

\subsection{Intuition and Main Idea}
\label{subsec:intuition}
The main idea behind backtracking counterfactuals---and its core conceptual difference to interventional counterfactuals---is that the causal laws, not the background conditions, are shared between the factual and counterfactual worlds. Consequently, the exogenous variables need to be allowed to differ to accommodate possibly contradictory facts and counterfacts. 
This is illustrated in~\cref{subfig:BC}. 

The name ``backtracking'' refers to the act of updating upstream variables to explain counterfacts without breaking the causal laws. Such backtracking is not needed for interventional counterfactuals for which the modified equations take care of this.
Let us first explore this idea through an example.%
\begin{example}[Backtracking counterfactual]
\label{ex:BC}
Consider the same model~\eqref{eq:example_linear_SCM} and factual observation $(X=1, Y=2, Z=2)$ from~\cref{ex:IC}. Suppose that now we are instead interested in the {\em backtracking} counterfactual ``what would have been, had $Y$ instead been {\em observed} to be $Y^*=3$''.

From abduction on the factual observation, we obtain the same posterior $P(\Ub~|~X=1, Y=2, Z=2)$ being a point mass on $(U_X=1,U_Y=1, U_Z=-1)$. 
However, in contrast to~\eqref{eq:ex_IC}, the structural equations in the counterfactual world now take the form:
\begin{equation}
\label{eq:ex_BC}
    X^*:=U_X^*, \qquad\qquad 
    Y^*:=X^*+U_Y^*=3, \qquad\qquad
    Z^*:=X^*+Y^*+U_Z^*;
\end{equation}
We highlight the following two key differences between~\eqref{eq:ex_BC} and its interventional counterpart~\eqref{eq:ex_IC}: (i) the causal laws remain fully in place, as apparent from the dependence of $Y^*$ on $X^*$ and $U_Y^*$ in~\eqref{eq:ex_BC}; (ii) the exogenous variables are not (necessarily) shared, as apparent from the introduction of new counterfactual background variables $(U_X^*, U_Y^*, U_Z^*)$ in~\eqref{eq:ex_BC}. 

Now, in order for $Y^*:=X^*+U_Y^*=3$ to hold, at least one of $X^*$ or $U_Y^*$ need to differ from their factual values of $X=U_Y=1$.
In other words, with the causal laws in the counterfactual world still being in place, the change from $Y=2$ to $Y^*=3$ needs to be explained by a change in $X$, $U_Y$, or both. 
In fact, there is an infinite number of such changed $(X^*,U_Y^*)$ given by the solutions to $X^*+U_Y^*=3$. 
Let us explore a few simple options: 
\begin{enumerate}[label=(\roman*),left=-12pt .. \parindent]
    \item We
    explain the change in $Y$ entirely by a change in $U_Y$: this can be achieved by keeping $X^*$ and hence also $U_X^*$ equal to their factual values of one and setting $U_Y^*=2$. Then $U_Z^*$ could remain equal to its factual value of $U_Z=-1$, leading to $Z^*:=X^*+Y^*+U_Z^*=1+3-1=3$ which is identical to the interventional counterfactual computed in~\cref{ex:IC}.\footnote{Except for the detail that, here, $Y^*=2$ arises from a change in $U_Y$, instead of through an intervention $do(Y^*=2)$ as in~\cref{ex:IC}. However, the resulting implications for the endogenous variables are the same.}
    \item We
    explain the change in $Y$ entirely by a change in $X^*$: we keep $U_Y^*=U_Y=1$ and set $X^*=2$. 
    As a consequence, for the causal law $X^*:=U_X^*$ to still be satisfied, we also need to adjust $U_X$ and set $U_X^*=2$.
    Again, we can keep $U_Z^*=U_Z=-1$, leading to $Z^*:=X^*+Y^*+U_Z^*=2+3-1=4$ which differs from the interventional counterfactual computed in~\cref{ex:IC}! 
    
    \item
    We explain the change in $Y$ by a change in both $X^*$ and $U_Y^*$: for this we set $U_Y^*=X^*=1.5$, adjust $U_X$ accordingly to $U_X^*=1.5$, keep $U_Z^*=U_Z=-1$, and obtain $Z^*:=X^*+Y^*+U_Z^*=1.5+3-1=3.5$, which is yet another outcome for $Z^*$. 
\end{enumerate}
\end{example}
Before returning to~\cref{ex:BC} later, let us summarise some of the points it illustrates already:

First, to explain the change in $Y^*\neq Y$ while keeping the causal laws intact, we had to change at least one of the (endogenous or exogenous) parents of $Y$, that is, to propagate the change upstream in the causal hierarchy---this is exactly what {\em backtracking} refers to (cf.~Remark~\ref{remark:forward_tracking}).

\looseness-1 Second, unlike for the interventional counterfactual in~\cref{ex:IC}, the backtracking counterfactual in~\cref{ex:BC} does not lead to a unique solution.
Since the background variables are allowed to differ, there can be multiple ways of setting $\Ub^*\neq\Ub$ that satisfy the causal laws and agree with the counterfactual antecedent.
When changing only the exogenous variable $U_Y^*$ corresponding to the counterfactual antecedent $Y^*$ as in case (i), we reached the same conclusion as when interpreting such change in interventional terms as in~\cref{ex:IC}.
When changing an ancestor that influences a descendant ($X^*\to Z^*$) as in cases (ii) and (iii), however, we reached different conclusions from~\cref{ex:IC}.
This highlights that backtracking counterfactuals may or may not agree with their interventional counterparts. 
Note also that it was not necessary to change the exogenous variable $U_Z^*$ of the descendant variable $Z$ of $Y$ from its factual value $U_Z$, even though this would in principle also be possible.

Recall that in the factual world
we have $(U_X,U_Y,U_Z)=(1,1,-1)$. We then constructed three counterfactual worlds that are consistent with $Y^*=2$ and the causal laws:\footnote{In fact, the (infinite) set of all valid counterfactual worlds is given by all $(U_X^*,3-U_X^*,U_Z^*)$ in the domain of $\Ub$, with corresponding predictions $(X^*,Z^*)=(U_X^*, 3+U_X^*+U_Z^*)$.}
\begin{equation}
\label{eq:ex_BC_worlds}
    (U_X^*,U_Y^*,U_Z^*)=
    \begin{cases}
        (1, 2, -1) &\implies \quad (X^*,Z^*)=(1,3) \qquad\quad  \text{for case (i)},\\
        (2, 1, -1)  &\implies \quad (X^*,Z^*)=(2,4) \qquad \quad \text{for case (ii)},\\
        (1.5, 1.5, -1) &\implies \quad (X^*,Z^*)=(1.5,3.5) \quad \,\,\, \text{for case (iii)}.
    \end{cases}
\end{equation}
In order to pick one of these worlds, or form a weighted average of their predictions, we require a notion of preference between worlds, such as a closeness or similarity measure. 
This will be our starting point for formalising the semantics of backtracking counterfactuals.

\subsection{Formal Semantics}
\label{subsec:formalisation}
In line with the probabilistic causal modelling framework, we now introduce the missing object for formal probabilistic reasoning about backtracking counterfactuals: an a priori similarity measure between factual and counterfactual worlds in the form of a {\em backtracking conditional} $P_B(\Ub^*~|~\Ub)$. 
This is a collection of probability functions $P_B(\Ub^*~|~\Ub=\ub)$, one for each $\ub$ in the domain of $\Ub$, quantifying the likelihood of each counterfactual world $\ub^*$ given factual world $\ub$,\footnote{Since the causal laws, and thus the causal model~$\Mcal$, remain unchanged for backtracking counterfactuals, we will also simply refer to instantiations of $\Ub$ as ``worlds'' in this context.} prior to any factual (or counterfactual) observations.
This provides us with a very general framework to encode various notions of cross-world similarity, depending on the context and the involved variables' domains.
We remain agnostic to the exact choice of~$P_B(\Ub^*~|~\Ub)$ for now and defer further discussion of desiderata, properties, and specific choices thereof to~\cref{subsec:choices}.

Together with the prior $P(\Ub)$, the backtracking conditional $P_B(\Ub^*~|~\Ub)$ induces a {\em joint distribution over factual and counterfactual worlds} given by:
\begin{equation}
    \label{eq:joint_world_dist}
    P_B(\Ub^*, \Ub) = P(\Ub)P_B(\Ub^*~|~\Ub)
\end{equation}
We can now define the (joint) probability of backtracking counterfactuals similarly to~\eqref{eq:joint_prob_ICs}.
\begin{definition}[Probability of backtracking counterfactuals]
Let $(\Mcal,P(\Ub))$ be a probabilistic causal model and $P_B(\Ub^*~|~\Ub)$ a backtracking conditional, where $\Ub^*$ denotes the counterfactual version of $\Ub$ defined over the same domain. 
For any not necessarily disjoint subsets of variables $\Yb,\Zb\subseteq \Vb$ and any realisations $\yb^*,\zb$ thereof, the probability of backtracking counterfactuals is given by:
\begin{equation}
\label{eq:joint_prob_BCs}
   P_B(\yb^*,\zb):= P_B(\Yb^* = \yb^*, \Zb = \zb)=\sum_{(\ub^*,\ub)} P_B(\ub^*,\ub) \, \1_{\{\Yb^*(\ub^*)=\yb\}} \, \1_{\{\Zb(\ub) = \zb\}}.
\end{equation}
\end{definition}

Any other quantities of interest can then be derived from~\eqref{eq:joint_prob_BCs} by standard probabilistic inference (that is, through marginalisation and conditioning).
In particular, we can now answer backtracking or observational counterfactuals of the form: ``given that we factually observed $\Zb$ to be~$\zb$, what would be the probability that $\Yb$ would be~$\yb^*$, had we \textit{observed} $\Xb$ to be~$\xb^*$?''.
Provided that $P_B(\xb^*,\zb)>0$, we obtain the corresponding expression $P_B(\yb^*~|~\xb^*,\zb)$ through the following three-step procedure, which loosely mirrors that for interventional counterfactuals given in~\cref{sec:background}:
\begin{enumerate}[left=0pt .. \parindent]
    \item \textbf{Cross-World Abduction:}  Update $P_B(\Ub^*,\Ub)$ from~\eqref{eq:joint_world_dist} by the evidence $(\xb^*,\zb,)$ to obtain the joint (``cross-world'') posterior $P(\Ub^*,\Ub~|~\xb^*,\zb)$ given by
    \begin{equation}
    \label{eq:cross-world-abduction}
        P_B(\ub^*, \ub~|~\xb^*,\zb)=\frac{P_B(\ub^*, \ub)}{P_B(\xb^*,\zb)}
        \, \1_{\{\Xb^*(\ub^*)=\xb^*\}}
        \, \1_{\{\Zb(\ub)=\zb\}}
    \end{equation}
    where 
    $P_B(\xb^*,\zb)$ is given by~\eqref{eq:joint_prob_BCs}.\footnote{This step can also be viewed as the projection of $P_B(\Ub^*,\Ub)$ onto the subspace consistent with $(\xb^*,\zb)$.}
    \item \textbf{Marginalisation:} Marginalise out $\Ub$ to obtain the counterfactual posterior $P_B(\Ub^*~|~\xb^*,\zb)$:
    \begin{equation}
    \label{eq:marginalisation}
        P_B(\ub^*~|~\xb^*,\zb) =\sum_\ub P_B(\ub^*, \ub~|~\xb^*,\zb).
    \end{equation}
    \item \textbf{Prediction:} Use the model $(\Mcal, P_B(\Ub^*~|~\xb^*,\zb))$ to predict $\Yb^*$:
\begin{equation}
   P_B(\yb^*~|~\xb^*,\zb)
   =\sum_{\ub^*} P_B(\ub^*~|~\xb^*,\zb) \,\1_{\{\Yb^*(\ub^*)=\yb^*\}}
\end{equation}
\end{enumerate}
\begin{remark}\label{remark:nosolution}
Whereas for interventional counterfactuals the antecedent is always attainable by intervention, backtracking counterfactuals may have no solution if the antecedent is incompatible with the causal laws.
For example, if $Y:=X$, any backtracking counterfactual with antecedent $y^*\neq x^*$ 
cannot be realized. 
Dealing with such ``counterlegals''~\citep{fisher2017counterlegal} appropriately
requires a semantics that combines  backtracking and interventions. 
We return to this in~\cref{app:unified}.
\end{remark}
To illustrate this procedure, we now carry out the respective calculations for~\cref{ex:BC}.
\addtocounter{example}{-1}
\begin{example}[continued]
We now compute~$P_B(X^*,Z^*~|~Y^*=3,X=1,Y=2,Z=2)$.
Suppose that we use the backtracking conditional $P_B(\Ub^*~|~\Ub=\ub)=\Ncal(\ub,\text{diag}(\sigma_X^2, \sigma_Y^2,\sigma_Z^2))$ which assigns higher probability to values of $\Ub^*$ that are close to a given $\ub$, with the relative scale of closeness given by $\sigma_X^2, \sigma_Y^2,\sigma_Z^2$ (see~\cref{subsec:choices} for more details).
\begin{enumerate}[left=0pt .. \parindent]
    \item \textbf{Cross-World Abduction:} 
    Ignoring multiplicative constants w.r.t.\ $(\ub^*,\ub)$, we find
    \begin{equation}
        P_B(\ub^*, \ub~|~Y^*=3,X=1,Y=2,Z=2)\propto P_B(\ub^*, \ub)
        \,\1_{\{\ub{=}(1,1,-1)\}}
        \,\1_{\{u_Y^*=3-u_X^*\}}
    \end{equation}
    \item \textbf{Marginalisation:} Only one
    non-zero term in the sum over $\ub$'s in~\eqref{eq:marginalisation} remains:
    \begin{equation}
    \begin{aligned}
        P_B(\ub^*~|~Y^*=3,X=1,Y=2,Z=2)\propto P_B(\ub^*~|~\ub=(1,1,-1)) 
        \,\1_{\{u_Y^*=3-u_X^*\}}\\
\propto \exp\left\{
        -\frac{(u_X^*-1)^2}{2\sigma_X^2}
        -\frac{(3-u_X^*-1)^2}{2\sigma_Y^2}
        -\frac{(u_Z^*+1)^2}{2\sigma_Z^2}
        \right\}\,\1_{\{u_Y^*=3-u_X^*\}}
    \end{aligned}
    \end{equation}
    Completing the square for $u_X^*$, we finally obtain the counterfactual posterior:
    \begin{equation*}
        \label{eq:ex_posterior}
        \Ub^*~|~(Y^*{=}3,X{=}1,Y{=}2,Z{=}2) \sim
        \Ncal_{(U_X^*,U_Z^*)}
        \left(
        \begin{pmatrix}
        1+\frac{r}{1+r} \\ -1
        \end{pmatrix},
        \begin{pmatrix}
         \frac{\sigma_X^2}{1+r} & 0 \\
         0 & \sigma^2_Z
         \end{pmatrix}
        \right)
        \delta(U_Y^*{=}3-U_X^*)
    \end{equation*}
    where $r=\frac{\sigma_X^2}{\sigma_Y^2}$ denotes the ratio of variances from the backtracking conditional.
    \item \textbf{Prediction:} 
    Using the relation $(X^*,Z^*)=(U_X^*, 3+U_X^*+U_Z^*)$ derived from the causal laws, we finally obtain the solution as a linear transformation of the Gaussian from step 2.:
    \begin{equation}
    \label{eq:ex_BC_CF_posterior}
        (X^*,Z^*)~|~(Y^*{=}3,X{=}1,Y{=}2,Z{=}2)
        \sim
        \Ncal
        \left(
        \begin{pmatrix}
        1+\frac{r}{1+r} \\ 3+\frac{r}{1+r}
        \end{pmatrix},
        \begin{pmatrix}
         \frac{\sigma_X^2}{1+r} & \frac{\sigma_X^2}{1+r} \\
         \frac{\sigma_X^2}{1+r} & \sigma^2_Z+ \frac{\sigma_X^2}{1+r}
         \end{pmatrix}
        \right)
    \end{equation}
\end{enumerate}
This form of~\eqref{eq:ex_BC_CF_posterior} is consistent with intuition: if changes in $U_X^*$ and $U_Y^*$ are considered equally likely (i.e., for $\sigma_X^2=\sigma_Y^2$ in the backtracking conditional $P_B(\Ub^*~|~\Ub)$), we have $r=1$ and case (iii) in~\eqref{eq:ex_BC_worlds} is the most likely (maximum a posteriori) scenario; if, on the other hand, $U_X^*$ is considered much less likely to change (i.e., $\sigma_X^2\ll \sigma_Y^2$ and $r\to 0$), \eqref{eq:ex_BC_CF_posterior} centers around case (i); conversely, if $U_Y^*$ is much less likely to change (i.e., $\sigma_X^2\gg \sigma_Y^2$ and $r\to \infty$), \eqref{eq:ex_BC_CF_posterior} centers around case (ii).
\end{example}

\subsection{On the Choice of the Backtracking Conditional~\texorpdfstring{$P_B(\Ub^*~|~\Ub)$}{P(U*|U)}}
\label{subsec:choices}
As we have seen, $P_B(\Ub^*~|~\Ub)$ plays an important role in answering backtracking counterfactuals: different choices may yield different answers.
We therefore now discuss desiderata for, and relevant properties of, the backtracking conditional~$P_B(\Ub^*~|~\Ub)$.

First, an intuitive desideratum is that $P_B(\Ub^*~|~\Ub)$ should assign high probability to values of $\Ub^*$ that are considered close to a given $\ub$, and small probability to those that are considered far.
\begin{property}[Preference for Closeness]
\label{prop:closeness}
We say that $P_B(\Ub^*~|~\Ub)$ has {\em Preference for Closeness} if
\begin{equation}
\label{eq:mode}
    \forall \ub: \, \argmax_{\ub^*} P_B(\ub^*~|~\ub) = \{\ub\}.
\end{equation}
\end{property}
Note that we could also impose such preference for closeness w.r.t.\ the mean or median, instead of (or in addition to) the mode as in~\eqref{eq:mode}. 
Second, to avoid a-priori asymmetries between the factual and counterfactual world,
the following symmetry requirement also seems natural.
\begin{property}[Symmetry]
\label{prop:symmetry}
We say that $P_B(\Ub^*~|~\Ub)$ is {\em symmetric} if for any $(\ub^*,\ub)$:
\begin{equation}
\label{eq:symmetry}
    P_B(\ub^*~|~\ub)=P_B(\ub~|~\ub^*)
\end{equation}
\end{property}
Note that this implies and is implied by matching marginals, or priors, across worlds:
\begin{lemma}[Symmetry is equivalent to matching marginals]
$P_B(\Ub^*~|~\Ub)$ is symmetric if and only if the marginal of $\Ub^*$ induced by $P(\Ub)$ and $P_B(\Ub^*~|~\Ub)$ matches $P(\Ub)$, that is:
\begin{equation}
    P_B(\Ub^*):=\sum_\ub P_B(\Ub^*~|~\ub)P(\ub)=P(\Ub).
\end{equation}
\end{lemma}
\begin{proof}
The result follows from~\eqref{eq:symmetry}, Bayes rule, and the equality of domains of $\Ub$ and $\Ub^*$.
\end{proof}
Third, since it is often assumed that the exogenous variables are mutually independent (i.e, that $P(\Ub)$ factorises), a similar natural property for the backtracking conditional is that each counterfactual variable $U^*_j$ only depends on its factual counterpart $U_j$.
\begin{property}[Decomposability]
\label{prop:decomposability}
We say that $P_B(\Ub^*~|~\Ub)$ is {\em decomposable} if for any $(\ub^*,\ub)$:
\begin{equation}
    P_B(\ub^*~|~\ub)=\prod_{j=1}^m P_B(u_j^*~|~u_j).    
\end{equation}
\end{property}

In line with the philosophical notion of \textit{closest possible worlds}, one option is to construct $P_B(\Ub^*~|~\Ub)$ based on a distance function $d(\cdot,\cdot)$ defined over the shared domain of $\Ub$ and $\Ub^*$:
\begin{equation}
\label{eq:distance}
    P_B(\ub^*~|~\ub)=\frac{1}{Z}\exp\{-d(\ub^*,\ub)\}
\end{equation}
where $Z=\sum_{\ub^*}\exp\{-d(\ub^*,\ub)\}$ is a normalisation constant.
\begin{lemma}
Any $P_B(\ub^*~|~\ub)$ of the form~\eqref{eq:distance} satisfies Properties~\ref{prop:closeness} and \ref{prop:symmetry}.
Further, if 
    $d(\ub^*,\ub)$ can be written as $\sum_j d(u_j^*,u_j)$, then
$P_B(\ub^*~|~\ub)$ also satisfies Property~\ref{prop:decomposability}.
\end{lemma}
\begin{proof}
By definition, any distance satisfies $d(\ub^*,\ub)\geq 0$ with equality iff\ $\ub^*=\ub$ (implying Property~\ref{prop:closeness}) and $d(\ub^*,\ub)=d(\ub,\ub^*)$ (implying Property~\ref{prop:symmetry}).
Property~\ref{prop:decomposability} follows from substituting the sum of component-wise distances
into~\eqref{eq:distance} and writing it as a product of exponentials.
\end{proof}

For real valued $\Ub\in\mathbb{R}^m$, a natural choice is the (squared) Mahalanobis distance
\begin{equation}
    d(\ub^*,\ub)=\frac{1}{2}(\ub-\ub^*)^\top\Sigmab^{-1}(\ub-\ub^*)
\end{equation}
for some positive-definite, symmetric (covariance) matrix $\Sigmab$, giving rise to the multivariate \textit{Gaussian backtracking conditional} $P_B(\Ub^*|\Ub=\ub)=\Ncal(\ub,\Sigmab)$ (for which mode, median and mean coincide).
If $\Sigmab=\text{diag}(\sigma_1^2, ..., \sigma_m^2)$ as used in~\cref{ex:BC}, this satisfies Properties~\ref{prop:closeness}, \ref{prop:symmetry}, and \ref{prop:decomposability}.

We emphasize that the distance-based~\eqref{eq:distance} is but one option for specifying $P_B$. A simple alternative is to dismiss entirely what actually happened and take $ P_B(\ub^*~|~\ub)=P(\ub^*)$, 
corresponding to the extreme view that what would have happened is completely independent of the factual events, and is instead determined entirely by the prior probability. More generally, one could combine both options by using  parameters $\alpha,\beta\geq 0$ to weigh off the relative importance of prior and distance,
\begin{equation}
\label{eq:generalized}
    P_B(\ub^*~|~\ub)=\frac{1}{Z}P(\ub^*)^{\alpha}\exp\{-\beta d(\ub^*,\ub)\}\,,
\end{equation}
where $Z=\sum_{\ub^*}P(\ub^*)^{\alpha}\exp\{-\beta d(\ub^*,\ub)\}$ is a normalisation constant.

\subsection{Theoretical Insights}
\label{subsec:theory}
Our first theoretical insight is that exogenous non-ancestors of factual and counterfactual observations remain unaffected, in that their posterior is equal to their prior.
\begin{proposition}
Let $\Mcal$ be an acyclic causal model. Suppose that $P(\Ub)$ factorises and that $P_B(\Ub^*~|~\Ub)$ is decomposable. 
For any $\Xb,\Zb\subseteq \Vb$ with (endogenous and exogenous) ancestors $\mathbf{AN}_\Xb,\mathbf{AN}_\Zb$:
\begin{equation}
    \Ub_J\subseteq \Ub\setminus \left( \mathbf{AN_\Xb\cup\mathbf{AN}_\Zb}\right) \implies \forall (\xb^*,\zb): P(\Ub^*_J,\Ub_J~|~\xb^*,\zb)=\prod_{j\in J}P(U^*_j,U_j)
\end{equation}
\end{proposition}
\begin{proof}
The result follows from d-separation between $(\Ub^*_J,\Ub_J)$ and $(\Xb^*,\Zb)$.
\end{proof}
Our second insight concerns the distinguishability of SCMs based on backtracking counterfactuals. 
As is well-known, identifying probabilities of \textit{interventional} counterfactuals requires full knowledge of the laws $\mathbf{F}$, which is hard to obtain in practice~\citep{bareinboim2022pearl}. Interestingly, probabilities of \textit{observational} counterfactuals are not as demanding since they only depend on the solution  or ``reduced form'' $\Vb(\Ub)$~\citep[see, e.g.,][\S~10 for details]{scholkopf2022statistical}.
\begin{proposition}
\label{pro:reduced}
Let $(\Mcal_1,P(\Ub))$ and $(\Mcal_2,P(\Ub))$ be causal models over the same variables $\Vb$ and $\Ub$ whose laws $\mathbf{F}_1$ and $\mathbf{F}_2$ have 
identical solutions $\mathbf{V}_1(\mathbf{u})=\mathbf{V}_2(\mathbf{u})$ for all values $\mathbf{u}$. Then for any choice of $P_B(\Ub^*~|~\Ub)$, 
both models will imply identical $P_B(\yb^*,\zb)$ for all choices $\yb^*$ and $\zb$.
\end{proposition}
\begin{proof}
This is an immediate consequence of \eqref{eq:joint_prob_BCs} only involving \eqref{eq:joint_world_dist} and the solution function.
\end{proof}
In particular, this result implies that causal structure is not discernible purely based on backtracking counterfactuals, as demonstrated by the following example.
\begin{example}
\label{ex:reduced_form}
Consider the following three causal models over $\Vb=\{X,Y\}$ and $\Ub=\{U\}$ with laws $\Fb_{X\to Y}=\{X:=U,Y:=X\}$, $\Fb_{Y\to X}=\{Y:=U,X:=Y\}$, $\Fb_{X\leftrightarrow Y}=\{X:=U,Y:=U\}$, and some shared $P(\Ub)$. Then all share the same reduced form $(X,Y)=(U,U)$ and hence the same backtracking counterfactuals, despite differing in their causal diagrams. 
\end{example}

\section{Related Work}
\label{sec:related_work}
Most closely related to our formalisation is the Extended Structural Model proposed by~\citet[][Eq.~1]{lucas2015improved} to model recent empirical findings on the context-dependency of humans' use of backtracking. It employs a  decomposable backtracking conditional for Boolean variables, 
\begin{equation}
\label{eq:lucas}
    P_B(U_j^*~|~U_j)= s\delta(U_j) + (1-s)P(U_j^*),
\end{equation}
\looseness-1 where the ``stability'' hyperparameter $s\in[0,1]$ interpolates between independent worlds~($s=0$) and perfectly shared exogenous variables ($s=1$). However, forcing $\Ub^*$ to either copy $\Ub$ or to completely ignore it is quite restrictive, see~\eqref{eq:generalized} for a more flexible implementation of this idea.
Other existing accounts are less formal:
they mostly
consider non-probabilistic Boolean conditionals 
of the form ``Given world $w$,
if $\xb^*$ were true then $\yb^*$ would be true'' and typically
involve minimising the number of (endogenous and) exogenous variables that change across worlds. 
%
%
\cite{hiddleston2005causal}'s semantics focuses on minimising the number of 
 exogenous
non-descendants of the antecedent that change values, which corresponds to using a variant of the Manhattan distance in~\eqref{eq:distance}.
Interestingly, Pearl's interventional semantics~(\cref{sec:background}) can be interpreted in a similar vein, instead minimising the number of intervened-upon variables.
In fact, he uses this formulation to argue that this 
is essentially equivalent to \citeauthor{lewis1979counterfactual}'s possible world semantics \citep[p.241]{pearl2009causality}. 
\cite{fisher2017counterlegal} 
combines both approaches by applying \citeauthor{hiddleston2005causal}'s condition if there exists a backtracking solution and resorting to a minimal number of interventions only if necessary. \cite{lee2017hiddleston} also considers both semantics, but does not combine them into a single one. 
%
%

%
%
%
%
%

%
%
%
%
%
%

%
%

%
Empirical research by cognitive scientists confirms that depending on the context, humans interpret counterfactuals as 
either backtracking or non-backtracking: 
``participants are more likely to backtrack when explicitly asked to consider a counterfactual’s \textit{causes}. However, when directly asked about the \textit{effects} of a counterfactual state, most people don’t backtrack''~\citep{gerstenberg2013back}.
Moreover, people use the broader context, exact wording, and level of determinism of the involved mechanism ``to infer how the antecedent is most likely to have come about'' and use this to decide whether to backtrack or not~\citep{rips2013inference}.
It has also been suggested that backtracking is preferred if and only if doing so makes a counterfactual claim true~\citep{han2014conditions}. 

When evaluating counterfactual explanations for specific historical events, historians 
rely on the {\em minimal rewrite rule}. This rule requires making minimal changes to the actual world to create the necessary conditions that would have led to a particular counterfactual antecedent~\citep{tetlock1996counterfactual}. Reiss has argued that historians interpret minimality in a backtracking sense, 
stating 
that ``the antecendent is not implemented by a miracle'', but rather ``counterfactuals in history are backtracking''~\citep[p. 719]{reiss2009counterfactuals}. 
His 
informal description closely aligns with our formal semantics, 
as he states that counterfactuals are evaluated based 
on ``causal generalizations'' that must not be violated, and only those counterfactual background conditions that ``were likely'' are to be considered. Additionally, 
using models similar to those from~\cref{ex:reduced_form},
 Reiss observes that these historical counterfactuals cannot distinguish between the causal structure of a chain of causes and a common cause structure, a finding that is entirely consistent with our Proposition~\ref{pro:reduced}.
%
%
%

%

%
%
%
%

%

%

%
%
\section{Connections and Applications to Explainable AI (XAI)}\label{sec:xai}
\label{sec:connections_XAI}
We believe that backtracking counterfactuals may hold great promise for XAI, which
 is 
concerned with offering explanations for the decisions
of a
machine learning model implementing a function $Y=f(\Xb)$, where $\Xb$ are 
{\em input features} and $Y$ is the {\em output or label}. Given that we observe 
some input $\xb$ and corresponding output $y$, the general aim is to find a 
feature subset $\Zb\subseteq \Xb$ 
that ``explains'' $y=f(\xb)$.
A particularly promising approach has been to look for so-called nearest {\em counterfactual explanations}, meaning that we look for both $\zb$ {\em and} $\zb^*$ such that changing $\Zb$ from~$\zb$ to~$\zb^*$ would have resulted in 
$y^* \neq y$, and $\zb$  and $\zb^*$ are close according to some distance function~\citep{wachter2017counterfactual}.
Although the general idea is rather intuitive, there has been much discussion on how exactly to implement it. The core problem lies in the choice of values $\wb^*$ for the remaining features $\Wb^* = \Xb^* \setminus \Zb^*$ that should accompany $\zb^*$ so that $f(\zb^*,\wb^*)=y^*$.
%

This discussion between AI-researchers can in fact be seen as the operational analogue of the philosophical discussion on the semantics of counterfactuals that we started out with in~\cref{sec:intro}, and which is far from resolved. Hence, we are faced with the same dilemma as before: to backtrack, or not to backtrack? Interestingly, the 
proposal of~\citet{wachter2017counterfactual} does neither, keeping
all other variables
fixed at their factual values~$\wb^*=\wb$. This proposal has recently come under criticism for its failure to take into account causal dependencies. %
On the one hand, 
\cite{beckers22a} criticises  it
for failing to follow Pearl (and the entire causal modelling tradition) in maintaining an \textit{interventionist} reading of counterfactuals. On the other hand, \cite{mahajan2019preserving,crupi2022counterfactual} criticise it for ignoring violations of the causal laws, and offer alternative proposals that---without making it explicit---can be viewed as building on a kind of \textit{backtracking} semantics instead, though ones that involve distances between endogenous variables or require knowledge of the causal laws.

When counterfactual explanations are offered in the context of {\em algorithmic recourse}~\citep{ustun2019actionable}, meaning that one is looking for actionable changes to the input features that result in a more favourable outcome, the first criticism is undoubtedly on the right track~\citep{karimi2020algorithmic,karimi2021algorithmic}. For example, if an agent is told to increase their income by $\$5,000$ in order to be granted a loan, then this explanation only succeeds in correctly predicting the decision if it has correctly taken into account the forward-tracking consequences of this change on any features that are downstream of income, such as savings.
If these changes are implemented externally (thereby overriding the observational distribution), they should indeed be modelled as interventions.

Yet, in most cases the use of interventional counterfactuals is unrealistic, since they cannot be identified solely from the observational distribution that ML-methods are trained on. Moreover, it is not at all clear that interventional counterfactuals are the right approach in the context of explanations that aim to {\em contest}, or to simply {\em understand} or {\em diagnose}, the outcome that was reached. In line with the second criticism (and the recent trend in the philosophy literature), we claim that backtracking counterfactuals deserve attention as a promising alternative for these contexts, and we have here developed the tools to do so. Importantly, 
as
Proposition \ref{pro:reduced} suggests,  backtracking counterfactuals should be easier to identify, as they depend on the causal dependencies between endogenous variables only through the reduced form.
By leaving the causal laws intact,
they also remain closer to the observational distribution, or data manifold, which is a desirable constraint for counterfactual explanations~\citep{wexler2019if,joshi2019towards,poyiadzi2020face,sharma2020certifai}.

An implementation of this idea could look roughly  as follows. 
Let $(\Mcal,P(\Ub))$ be a probabilistic causal model with endogenous variables $\Vb=\Xb \cup \{Y\}$, $P(\Ub)$ having full support, and laws such that $Y=f(\Xb)$ holds with probability one.
For a particular choice of $P_B(\Ub^*~|~\Ub)$, we then say that ``$\xb$ rather than $\xb^*$ explains why $f(\xb)=y$ rather than $y^* \neq y$'' if such a change to $y^*$ would be most likely to have come about through $\xb^*$, that is, if $\xb^*\in\operatorname{arg\,max}_{\xb^*} P_B(\xb^*~|~y^*,\xb,y)$.
Note that this is exactly the \textit{diagnostic} kind of reasoning that only backtracking allows,
thus establishing a clear link between counterfactual explanations and maximum a posteriori backtracking counterfactuals.

In practice, interpretable explanations should be concise, highlighting just a few features that are very likely to have been different.
To implement this, we could
set a probability threshold $\alpha$ and a maximal number $k$ of ``explanatory'' features,
and then look for an optimal feature subset $\Zb\subseteq\Xb$
with $|\Zb| \leq k$ and altered counterfactual feature values $\zb^*$ satisfying $z^*_i \neq z_i$ for {\em all} $Z_i \in \Zb$,
such that $P_B(\zb^*~|~y^*,\xb,y) > \alpha$.\footnote{For continuous variables, we could, for example, instead set a distance threshold $\epsilon$ and then look for a set $A$ of values that are at least distance $\epsilon$ away, $d(\zb^*,\zb)>\epsilon$, and satisfy $P_B(\zb^*\in A~|~y^*,\xb,y) > \alpha$.}
Lastly, note that a variant of the original proposal can be retrieved within our proposed backtracking framework by considering $\operatorname{arg\,max}_{\zb^*} P_B(\zb^*~|~\Wb^*=\wb,y^*,\xb,y)$ for all choices $\Zb\subseteq \Xb$,
where as before $\Xb=\Wb \cup \Zb$ and $z^*_i \neq z_i$ for {\em all} $Z_i \in \Zb$.

\paragraph{Causal Attribution Analysis.}
\looseness-1 Counterfactual explanations are also particularly relevant for causal attribution tasks such as 
{\em root cause analysis} of an outlier event~$Y=y$~\citep{budhathoki2022causal}. A central idea is that the exogenous variables in an SCM (the ``roots'') ultimately explain why $Y=y$.
The proposed method thus involves
\textit{keeping the causal laws intact} while varying the values of some exogenous variables according to a ``counterfactual distribution'', which is emphasised as ``the key ingredient''~\citep[p.~4]{budhathoki2022causal}. 
This can be viewed as a form of backtracking, albeit not formulated in those terms. 
Suppose we observe an outlier $Y=y$ together with $\Xb=\xb$, and---assuming invertibility of the reduced form $\mathbf{V}(\Ub)$---use this to infer $\Ub=\ub$. The main building block of~\citeauthor{budhathoki2022causal}'s method is to quantify how much  each $\Ub_i \subseteq \Ub$ contributed to the outlier event by computing the counterfactual probability of a similar (or more extreme) outlier, conditional on the factual $\Ub=\ub$ and on 
a subset $\Ub_S\subset\Ub$ being fixed across worlds while the remaining~$\Ub_{-S}=\Ub\setminus \Ub_S$ are resampled from the prior $P(\Ub)$, combined with Shapley-value based symmetrization. We can give a precise expression to this as $P_B(\tau(y^*) \geq \tau(y) ~|~ \Ub_{S}^*=\ub_{S}, \Ub=\ub)$ with backtracking conditional $P_B(\Ub^*~|~\Ub)=P(\Ub^*)$, where $\tau$ is a calibration function to allow for comparing outliers over different ranges. 
Similar ideas based on ``structure-preserving interventions'' that do not change the laws $\Fb$ but only resample subsets of the exogenous $\Ub$ have also been used by~\citet{janzing2021quantifying} for attributing quantities such as entropy or variance to other variables in the model. 

%

%
\section{Conclusion and Future Work}
\label{sec:future_work}
We have here presented the first formal account of backtracking counterfactuals within probabilistic causal models~(\cref{sec:backtracking}). 
Doing so in full generality required the introduction of a new object, the backtracking conditional, which quantifies a notion of similarity between worlds~(\cref{subsec:formalisation}).
This involves a design choice that can take into account the specific context and the modeller's background assumptions. In fact, many previous accounts of backtracking can be seen as specific choices of such a conditional~(\cref{sec:related_work}).
We have laid out some sensible desiderata and means to achieve them in~\cref{subsec:choices}.

\looseness-1 As stressed throughout, we do not view backtracking as a replacement of the interventionist account~(\cref{sec:background}) but rather as complementary to it. 
Our work thus also emphasizes the ambiguous nature of counterfactuals and their semantics, a point which we discuss further through a worked-out example in~\cref{app:firingsquad}.
Proposition~\ref{pro:reduced} and~\cref{ex:reduced_form} make clear that backtracking counterfactuals do not allow for discerning causal structure 
and are therefore of limited use for interventional reasoning.
At the same time, they are helpful for the kind of diagnostic reasoning that occurs in settings where interventions are inconsistent with human judgement~(\cref{sec:related_work}), require knowledge that is unavailable~(\cref{sec:connections_XAI}), or are perhaps not even meaningful. %
The latter shows up in causal fairness analysis~\citep{kusner2017counterfactual,kilbertus2017avoiding,von2020fairness,plecko2022causal}, which often focuses on such attributes (like race or gender), leading some to dispute their status as valid causes~\citep{holland2008causation,hu2020s}. %
For this reason, we consider applications of backtracking for fairness analysis an interesting future direction.

\looseness-1 By endowing backtracking with a formal semantics within the same general SCM framework also used for interventional counterfactuals, the present work paves the way for a unified framework of both backtracking \textit{and} non-backtracking counterfactual reasoning.
We make a proposal for incorporating hard interventions in~\cref{app:unified}, but future work is needed to provide a more comprehensive and rigorous account.
The general idea is to allow modifying both the causal laws and the background conditions, and to weigh off changes to both of them through an appropriate choice of backtracking conditional.
This would not only imply the property of being guaranteed a solution (recall Remark~\ref{remark:nosolution}), but may also allow for more accurate models of human counterfactual reasoning. 
%
%
%
%
%

%

%

%

%

%

%
%

%
%
%
%
%
%
%
%
%
  
%
%
%

%
%
%
%
%
%
%
%
\acks{The authors thank Dominik Janzing for insightful discussions, and the anonymous reviewers and the area chair for helpful comments and suggestions.
This work
was supported by the T\"ubingen AI Center
and the Deutsche Forschungsgemeinschaft (DFG, German Research Foundation) under Germany's Excellence Strategy – EXC number 2064/1 – Project number 390727645.}

\bibliography{references}

\begin{thebibliography}{43}
\providecommand{\natexlab}[1]{#1}
\providecommand{\url}[1]{\texttt{#1}}
\expandafter\ifx\csname urlstyle\endcsname\relax
  \providecommand{\doi}[1]{doi: #1}\else
  \providecommand{\doi}{doi: \begingroup \urlstyle{rm}\Url}\fi

\bibitem[Balke and Pearl(1994)]{BalkeP94}
Alexander Balke and Judea Pearl.
\newblock Probabilistic evaluation of counterfactual queries.
\newblock In \emph{Proceedings of the 12th National Conference on Artificial
  Intelligence}, pages 230--237. {AAAI} Press / The {MIT} Press, 1994.

\bibitem[Bareinboim et~al.(2022)Bareinboim, Correa, Ibeling, and
  Icard]{bareinboim2022pearl}
Elias Bareinboim, Juan~D Correa, Duligur Ibeling, and Thomas Icard.
\newblock On {P}earl’s hierarchy and the foundations of causal inference.
\newblock In \emph{Probabilistic and Causal Inference: The Works of Judea
  Pearl}, pages 507--556, 2022.

\bibitem[Beckers(2022)]{beckers22a}
Sander Beckers.
\newblock Causal explanations and {XAI}.
\newblock In \emph{Proceedings of the First Conference on Causal Learning and
  Reasoning}, volume 177, pages 90--109. PMLR, 2022.

\bibitem[Budhathoki et~al.(2022)Budhathoki, Minorics, Bl{\"o}baum, and
  Janzing]{budhathoki2022causal}
Kailash Budhathoki, Lenon Minorics, Patrick Bl{\"o}baum, and Dominik Janzing.
\newblock Causal structure-based root cause analysis of outliers.
\newblock In \emph{International Conference on Machine Learning}, pages
  2357--2369. PMLR, 2022.

\bibitem[Crupi et~al.(2022)Crupi, Castelnovo, Regoli, and San
  Miguel~Gonzalez]{crupi2022counterfactual}
Riccardo Crupi, Alessandro Castelnovo, Daniele Regoli, and Beatriz San
  Miguel~Gonzalez.
\newblock Counterfactual explanations as interventions in latent space.
\newblock \emph{Data Mining and Knowledge Discovery}, pages 1--37, 2022.

\bibitem[Dawid(2000)]{dawid2000causal}
Philip Dawid.
\newblock Causal inference without counterfactuals.
\newblock \emph{Journal of the American statistical Association}, 95\penalty0
  (450):\penalty0 407--424, 2000.

\bibitem[Dorr(2016)]{dorr2016against}
Cian Dorr.
\newblock Against counterfactual miracles.
\newblock \emph{The Philosophical Review}, 125\penalty0 (2):\penalty0 241--286,
  2016.

\bibitem[Esfeld(2021)]{esfeld2021super}
Michael Esfeld.
\newblock Super-humeanism and free will.
\newblock \emph{Synthese}, 198\penalty0 (7):\penalty0 6245--6258, 2021.

\bibitem[Fisher(2017{\natexlab{a}})]{fisher2017causal}
Tyrus Fisher.
\newblock Causal counterfactuals are not interventionist counterfactuals.
\newblock \emph{Synthese}, 194\penalty0 (12):\penalty0 4935--4957,
  2017{\natexlab{a}}.

\bibitem[Fisher(2017{\natexlab{b}})]{fisher2017counterlegal}
Tyrus Fisher.
\newblock Counterlegal dependence and causation’s arrows: Causal models for
  backtrackers and counterlegals.
\newblock \emph{Synthese}, 194\penalty0 (12):\penalty0 4983--5003,
  2017{\natexlab{b}}.

\bibitem[Gerstenberg et~al.(2013)Gerstenberg, Bechlivanidis, and
  Lagnado]{gerstenberg2013back}
Tobias Gerstenberg, Christos Bechlivanidis, and David~A Lagnado.
\newblock Back on track: Backtracking in counterfactual reasoning.
\newblock In \emph{Proceedings of the Annual Meeting of the Cognitive Science
  Society}, volume~35, 2013.

\bibitem[Han et~al.(2014)Han, Jimenez-Leal, and Sloman]{han2014conditions}
Jung-Ho Han, William Jimenez-Leal, and Steve Sloman.
\newblock Conditions for backtracking with counterfactual conditionals.
\newblock In \emph{Proceedings of the Annual Meeting of the Cognitive Science
  Society}, volume~36, 2014.

\bibitem[Hiddleston(2005)]{hiddleston2005causal}
Eric Hiddleston.
\newblock A causal theory of counterfactuals.
\newblock \emph{No{\^u}s}, 39\penalty0 (4):\penalty0 632--657, 2005.

\bibitem[Holland(2008)]{holland2008causation}
Paul~W Holland.
\newblock Causation and race.
\newblock \emph{White logic, white methods: Racism and methodology}, pages
  93--109, 2008.

\bibitem[Hu and Kohler-Hausmann(2020)]{hu2020s}
Lily Hu and Issa Kohler-Hausmann.
\newblock What's sex got to do with machine learning?
\newblock In \emph{Proceedings of the 2020 Conference on Fairness,
  Accountability, and Transparency}, pages 513--513, 2020.

\bibitem[Janzing et~al.(2021)Janzing, Bl{\"o}baum, Minorics, Faller, and
  Mastakouri]{janzing2021quantifying}
Dominik Janzing, Patrick Bl{\"o}baum, Lenon Minorics, Philipp Faller, and
  Atalanti Mastakouri.
\newblock Quantifying intrinsic causal contributions via structure preserving
  interventions.
\newblock \emph{arXiv 2007.00714}, 2021.

\bibitem[Joshi et~al.(2019)Joshi, Koyejo, Vijitbenjaronk, Kim, and
  Ghosh]{joshi2019towards}
Shalmali Joshi, Oluwasanmi Koyejo, Warut Vijitbenjaronk, Been Kim, and Joydeep
  Ghosh.
\newblock Towards realistic individual recourse and actionable explanations in
  black-box decision making systems.
\newblock \emph{arXiv preprint arXiv:1907.09615}, 2019.

\bibitem[Karimi et~al.(2020)Karimi, Von~K{\"u}gelgen, Sch{\"o}lkopf, and
  Valera]{karimi2020algorithmic}
Amir-Hossein Karimi, Julius Von~K{\"u}gelgen, Bernhard Sch{\"o}lkopf, and
  Isabel Valera.
\newblock Algorithmic recourse under imperfect causal knowledge: a
  probabilistic approach.
\newblock \emph{Advances in Neural Information Processing Systems},
  33:\penalty0 265--277, 2020.

\bibitem[Karimi et~al.(2021)Karimi, Sch{\"o}lkopf, and
  Valera]{karimi2021algorithmic}
Amir-Hossein Karimi, Bernhard Sch{\"o}lkopf, and Isabel Valera.
\newblock Algorithmic recourse: from counterfactual explanations to
  interventions.
\newblock In \emph{Proceedings of the 2021 ACM Conference on Fairness,
  Accountability, and Transparency}, pages 353--362, 2021.

\bibitem[Kilbertus et~al.(2017)Kilbertus, Rojas~Carulla, Parascandolo, Hardt,
  Janzing, and Sch{\"o}lkopf]{kilbertus2017avoiding}
Niki Kilbertus, Mateo Rojas~Carulla, Giambattista Parascandolo, Moritz Hardt,
  Dominik Janzing, and Bernhard Sch{\"o}lkopf.
\newblock Avoiding discrimination through causal reasoning.
\newblock \emph{Advances in neural information processing systems}, 30, 2017.

\bibitem[Kusner et~al.(2017)Kusner, Loftus, Russell, and
  Silva]{kusner2017counterfactual}
Matt~J Kusner, Joshua Loftus, Chris Russell, and Ricardo Silva.
\newblock Counterfactual fairness.
\newblock \emph{Advances in neural information processing systems}, 30, 2017.

\bibitem[Lee(2017)]{lee2017hiddleston}
Kok~Yong Lee.
\newblock Hiddleston’s causal modeling semantics and the distinction between
  forward-tracking and backtracking counterfactuals.
\newblock \emph{Studies in Logic}, 10\penalty0 (1), 2017.

\bibitem[Lewis(1973)]{lewis1973counterfactuals}
David Lewis.
\newblock \emph{Counterfactuals}.
\newblock Oxford: Blackwell Publishers and Cambridge, MA: Harvard University
  Press, 1973.

\bibitem[Lewis(1979)]{lewis1979counterfactual}
David Lewis.
\newblock Counterfactual dependence and time's arrow.
\newblock \emph{No{\^u}s}, pages 455--476, 1979.

\bibitem[Loewer(2007)]{Loewer2007-LOECAT}
Barry Loewer.
\newblock Counterfactuals and the second law.
\newblock In Huw Price and Richard Corry, editors, \emph{Causation, Physics,
  and the Constitution of Reality: Russell's Republic Revisited}. Oxford
  University Press, 2007.

\bibitem[Loewer(2020)]{loewer2020consequence}
Barry Loewer.
\newblock The consequence argument meets the mentaculus, 2020.
\newblock Working papers, Rutgers University.

\bibitem[Lucas and Kemp(2015)]{lucas2015improved}
Christopher~G Lucas and Charles Kemp.
\newblock An improved probabilistic account of counterfactual reasoning.
\newblock \emph{Psychological review}, 122\penalty0 (4):\penalty0 700, 2015.

\bibitem[Mahajan et~al.(2019)Mahajan, Tan, and Sharma]{mahajan2019preserving}
Divyat Mahajan, Chenhao Tan, and Amit Sharma.
\newblock Preserving causal constraints in counterfactual explanations for
  machine learning classifiers.
\newblock \emph{arXiv:1912.03277}, 2019.

\bibitem[Pearl(2009)]{pearl2009causality}
Judea Pearl.
\newblock \emph{Causality}.
\newblock Cambridge university press, 2009.

\bibitem[Pearl and Mackenzie(2018)]{pearl2018book}
Judea Pearl and Dana Mackenzie.
\newblock \emph{The book of why: the new science of cause and effect}.
\newblock Basic books, 2018.

\bibitem[Plecko and Bareinboim(2022)]{plecko2022causal}
Drago Plecko and Elias Bareinboim.
\newblock Causal fairness analysis.
\newblock \emph{arXiv preprint arXiv:2207.11385}, 2022.

\bibitem[Poyiadzi et~al.(2020)Poyiadzi, Sokol, Santos-Rodriguez, De~Bie, and
  Flach]{poyiadzi2020face}
Rafael Poyiadzi, Kacper Sokol, Raul Santos-Rodriguez, Tijl De~Bie, and Peter
  Flach.
\newblock Face: feasible and actionable counterfactual explanations.
\newblock In \emph{Proceedings of the AAAI/ACM Conference on AI, Ethics, and
  Society}, pages 344--350, 2020.

\bibitem[Reiss(2009)]{reiss2009counterfactuals}
Julian Reiss.
\newblock Counterfactuals, thought experiments, and singular causal analysis in
  history.
\newblock \emph{Philosophy of Science}, 76\penalty0 (5):\penalty0 712--723,
  2009.

\bibitem[Rips(2010)]{rips2010two}
Lance~J Rips.
\newblock Two causal theories of counterfactual conditionals.
\newblock \emph{Cognitive science}, 34\penalty0 (2):\penalty0 175--221, 2010.

\bibitem[Rips and Edwards(2013)]{rips2013inference}
Lance~J Rips and Brian~J Edwards.
\newblock Inference and explanation in counterfactual reasoning.
\newblock \emph{Cognitive Science}, 37\penalty0 (6):\penalty0 1107--1135, 2013.

\bibitem[Sch{\"o}lkopf and von K{\"u}gelgen(2022)]{scholkopf2022statistical}
Bernhard Sch{\"o}lkopf and Julius von K{\"u}gelgen.
\newblock From statistical to causal learning.
\newblock \emph{arXiv preprint arXiv:2204.00607}, 2022.

\bibitem[Sharma et~al.(2020)Sharma, Henderson, and Ghosh]{sharma2020certifai}
Shubham Sharma, Jette Henderson, and Joydeep Ghosh.
\newblock Certifai: A common framework to provide explanations and analyse the
  fairness and robustness of black-box models.
\newblock In \emph{Proceedings of the AAAI/ACM Conference on AI, Ethics, and
  Society}, pages 166--172, 2020.

\bibitem[Tetlock and Belkin(1996)]{tetlock1996counterfactual}
Philip~E Tetlock and Aaron Belkin.
\newblock \emph{Counterfactual thought experiments in world politics: Logical,
  methodological, and psychological perspectives}.
\newblock Princeton University Press, 1996.

\bibitem[Ustun et~al.(2019)Ustun, Spangher, and Liu]{ustun2019actionable}
Berk Ustun, Alexander Spangher, and Yang Liu.
\newblock Actionable recourse in linear classification.
\newblock In \emph{Proceedings of the conference on fairness, accountability,
  and transparency}, pages 10--19, 2019.

\bibitem[von K{\"u}gelgen et~al.(2022)von K{\"u}gelgen, Karimi, Bhatt, Valera,
  Weller, and Sch{\"o}lkopf]{von2020fairness}
Julius von K{\"u}gelgen, Amir-Hossein Karimi, Umang Bhatt, Isabel Valera,
  Adrian Weller, and Bernhard Sch{\"o}lkopf.
\newblock On the fairness of causal algorithmic recourse.
\newblock In \emph{Proceedings of the 36th AAAI Conference on Artificial
  Intelligence}, 2022.

\bibitem[Wachter et~al.(2017)Wachter, Mittelstadt, and
  Russell]{wachter2017counterfactual}
Sandra Wachter, Brent Mittelstadt, and Chris Russell.
\newblock Counterfactual explanations without opening the black box: Automated
  decisions and the {GDPR}.
\newblock \emph{Harv. JL \& Tech.}, 31:\penalty0 841, 2017.

\bibitem[Wexler et~al.(2019)Wexler, Pushkarna, Bolukbasi, Wattenberg,
  Vi{\'e}gas, and Wilson]{wexler2019if}
James Wexler, Mahima Pushkarna, Tolga Bolukbasi, Martin Wattenberg, Fernanda
  Vi{\'e}gas, and Jimbo Wilson.
\newblock The what-if tool: Interactive probing of machine learning models.
\newblock \emph{IEEE transactions on visualization and computer graphics},
  26\penalty0 (1):\penalty0 56--65, 2019.

\bibitem[Woodward(2021)]{woodward2021causation}
James Woodward.
\newblock \emph{Causation with a human face: Normative theory and descriptive
  psychology}.
\newblock Oxford University Press, 2021.

\end{thebibliography}

\tableofcontents

\clearpage
\appendix

\section{Toward A Unified Framework of Counterfactual Reasoning}\label{app:unified}
As alluded to in~\cref{sec:future_work}, combining backtracking and interventional counterfactuals into a single unified framework of counterfactual reasoning is an interesting problem, and we believe that our framework is suitable for making progress toward this goal. 
Below, we present a first attempt at doing so. 
Specifically, we show how hard interventions can be modelled through auxiliary variables, a change in the causal laws, and an appropriate choice of backtracking conditional.\footnote{We thank the area chair for suggesting this approach.} %

Let $(\Mcal, P(\Ub))$ be a probabilistic causal model with laws
\begin{equation}
\label{eq:app_laws}
    V_i:=f_i(\PA_i,\Ub_i)\qquad \qquad i=1,\dots, n.
\end{equation}
Suppose that $V_i$ takes values in $\mathcal{V}_i\subseteq \mathbb{R}$ and that we wish to reason about hard interventions of the form do$(V_i=\tilde{v}_i)$ for some constant $\tilde v_i\in\mathcal{V}_i$. 

First, we introduce an auxiliary regime variable $R_i$ taking values in $\mathcal{V}_i \cup \{\text{obs}\}$, where obs is a special place holder symbol that we use to denote the observational regime. 

Next, we replace the $i$\textsuperscript{th} law in~\eqref{eq:app_laws} with the following augmented law:
\begin{equation}
\label{eq:app_laws_ modified}
    V_i:=\tilde f_i(\PA_i,\Ub_i,R_i)=
    \begin{cases}
        f_i(\PA_i,\Ub_i) & \text{if} \,\, R_i=\text{obs}
        \\
        R_i & \text{otherwise}\,,
    \end{cases}
\end{equation}
that is, $V_i$ is determined by its original law in the observational regime ($R_i=\text{obs}$), and equal to $R_i$ outside it. Note that the latter precisely captures hard interventions through $R_i=\tilde v_i$.

Now suppose that we wish to reason about a counterfactual involving $V_i^*=v_i^*$ and that we want to encode the notion that this counterfactual value came about through an intervention. 
With the above construction, this simply amounts to conditioning on the event $R_i^*=v_i^*$, and then backtracking without further modifications as explained in the main text.

More generally, to combine both backtracking and interventional semantics, we can proceed as follows: First, we introduce an auxiliary $R_i$ and modify the causal laws as in~\eqref{eq:app_laws_ modified} for each $i$. Then, we consider the augmented set of exogenous variables $\Ub\cup \Rb$ with $\Rb=\{R_1, ..., R_n\}$, and place a prior on $\Rb$. For example, if we are interested in queries for which the factual world is purely observational, we can choose a $P(\Rb)$ that (independently of $\Ub$) puts all mass on \{obs\} for all $i$ (though other choices are, of course, also possible).
Finally, we specify a backtracking conditional $P_B(\Ub^*,\Rb^*~|~\Ub, \Rb)$ in which we may encode any available background knowledge on which of the causal laws are more likely to be violated through an intervention and which are more likely to remain intact.
We can then evaluate expressions of the form $$P(\Yb^*=\yb,\Zb=\zb~|~\Rb^*_j=\rb^*_j, \Rb_k=\rb_k)$$ where $\Rb_j^*,\Rb_k\subseteq \Rb$ (possibly empty) specify the assumed  observational or interventional regimes for some of the factual and counterfactual variables. 

\looseness-1 Note that a standard interventional counterfactual such as 
$P(Y_x = y ~|~ Z=z)$ can be retrieved as $$P(Y^* = y ~|~ Z=z, R^*_1=\text{obs}, ..., R^*_X=x, ..., R^*_n=\text{obs}, R_1=\text{obs}, \ldots, R_n=\text{obs}),$$ and similarly for more complex conditions. That is, we specify values for all the auxiliary variables
by setting them to obs for all factual variables and all counterfactual variables that are not intervened upon, and setting them to the appropriate intervention values for the other, intervened-upon variables (here only $X$). 
When combined with choosing $P_B(\ub^*~|~\ub)=\delta(\ub)$ so that the values of $\Ub$ are forced to be identical across worlds, we recover interventional counterfactuals as a special case in this unified framework.  

At the same time, we can of course also retrieve backtracking counterfactuals by setting all auxiliary variables, both counterfactual and factual, to obs and choosing whatever $P_B$ one likes. Therefore, our approach offers a unified framework for interpreting counterfactuals in SCMs.

\section{Firing Squad Example: Backtracking vs Interventions}\label{app:firingsquad}
We now illustrate the difference between interventional and backtracking approaches to interpreting and computing counterfactuals for the well-known firing squad example~\citep[][\S~7.1.2]{pearl2009causality}.

\paragraph{Setting.}
Unlike the examples from the main paper, the firing squad is a case of a deterministic causal model with binary variables. 
The setting is as follows: A captain $C$ of two riflemen $A$ and $B$ is waiting for a court order $U$ on whether a prisoner $P$ should be executed. If the court orders the execution ($U=1$), the captain signals ($C=1$), the two riflemen shoot ($A=B=1$), and the prisoner dies ($P=1$). 
More formally, we can express this scenario through the following SCM:
\begin{align}
    C&:=U \label{eq:captain}\\
    A&:=C \label{eq:riflemanA}\\
    B&:=C\\
    P&:=A\lor B
\end{align}
where $U\sim \mathrm{Bernoulli}(\theta)$ without loss of generality.

\begin{figure}[t]
    \centering
    \newcommand{\xshift}{5em}
    \newcommand{\yshift}{3.5em}
    \begin{subfigure}{0.5\textwidth}
    \begin{tikzpicture}
        \centering
        \node (U) [latent] {$U$};
        \node (C) [obs,yshift=-0.5*\yshift, xshift=-\xshift] {$C$};
        \node (A) [obs, yshift=-1.5*\yshift, xshift=-1.5*\xshift] {$A$};
        \node (B) [obs, xshift=-0.5*\xshift, yshift=-1.5*\yshift] {$B$};
        \node (P) [obs,  yshift=-2.5*\yshift,xshift=-1*\xshift]{$P$};
        \node (C*) [latent,yshift=-0.5*\yshift, xshift=\xshift] {$C^*$};
        \node (A*) [det, fill=gray!25, yshift=-1.5*\yshift, xshift=0.5*\xshift] {$A^*$};
        \node (B*) [latent, yshift=-1.5*\yshift, xshift=1.5*\xshift] {$B^*$};
        \node (P*) [latent,  yshift=-2.5*\yshift,xshift=\xshift]{$P^*$};
        \edge{U}{C,C*};
        \edge{C}{A,B};
        \edge{C*}{B*};
        \edge{A,B}{P};
        \edge{A*,B*}{P*};
    \end{tikzpicture}
    \caption{Interventional Interpretation: $P^*$ Dead}
    \label{fig:firingsquad_interventional}
    \end{subfigure}%
    \begin{subfigure}{0.5\textwidth}
    \begin{tikzpicture}
        \centering
        \node (U) [latent, xshift=-0.5*\xshift] {$U$};
        \node (U*) [latent, xshift=0.5*\xshift] {$U^*$};
        \node (C) [obs,yshift=-0.5*\yshift, xshift=-\xshift] {$C$};
        \node (A) [obs, yshift=-1.5*\yshift, xshift=-1.5*\xshift] {$A$};
        \node (B) [obs, xshift=-0.5*\xshift, yshift=-1.5*\yshift] {$B$};
        \node (P) [obs,  yshift=-2.5*\yshift,xshift=-1*\xshift]{$P$};
        \node (C*) [latent,yshift=-0.5*\yshift, xshift=\xshift] {$C^*$};
        \node (A*) [obs, yshift=-1.5*\yshift, xshift=0.5*\xshift] {$A^*$};
        \node (B*) [latent, yshift=-1.5*\yshift, xshift=1.5*\xshift] {$B^*$};
        \node (P*) [latent,  yshift=-2.5*\yshift,xshift=\xshift]{$P^*$};
        \edge{U}{C};
        \edge{U*}{C*};
        \edge{C}{A,B};
        \edge{C*}{A*,B*};
        \edge{A,B}{P};
        \edge{A*,B*}{P*};
        \path[<->,dashed] (U) edge (U*);
    \end{tikzpicture}
    \caption{Backtracking Interpretation: $P^*$ Alive}
    \label{fig:firingsquad_backtracking}
    \end{subfigure}
    \caption{\small 
    \textbf{The Firing Squad Example Highlights the Ambiguity of Counterfactuals.} We observe that the captain signals~($C=1$), the two riflemen shoot~($A=B=1$), and the prisoner dies~($P=1$). We want to answer the counterfactual question: \textit{what would have happened, had rifleman $A$ not shot~($A^*=0$)?}
    (a) The interventional interpretation is that the court still would have ordered the execution ($U=1$) and the captain still would have signalled ($C^*=1$) in the counterfactual world, but rifleman $A$ did not shoot ($A^*=0$) because they disobeyed the order or their rifle got jammed, thus locally violating the causal laws as indicated by the missing edge $C^*\to A^*$. Yet, rifleman $B$ still would have shot ($B^*=1$), and so the prisoner would still be dead ($P^*=0$).
    (b) In stark contrast, the backtracking interpretation is that, since the causal laws must remain untouched, $A^*=0$ could have only happened in the counterfactual world if the captain had not signalled  ($C^*=0$) because the court did not order the execution ($U^*=0$).
    Hence, rifleman $B$ also would not have shot ($B^*=0$) and so the prisoner would still be alive ($P^*=0$).
    }
    \label{fig:my_label}
\end{figure}
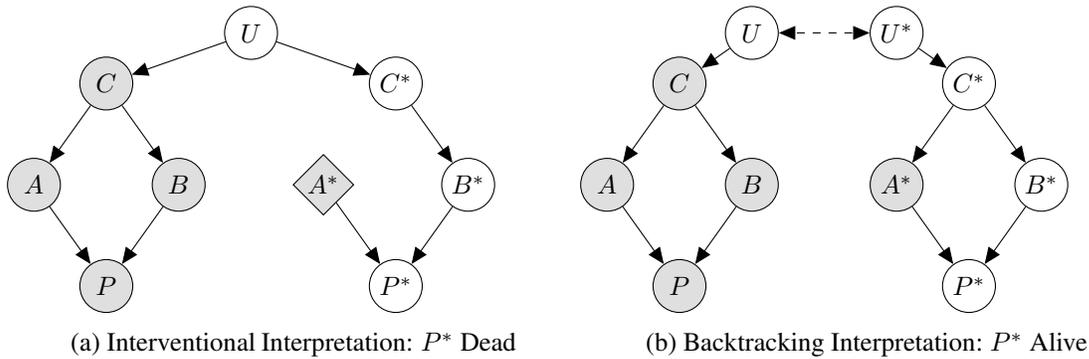

Suppose that we observe that $C=A=B=P=1$, that is, the captain signalled, both riflemen shot, and the prisoner died.
We now wish to answer the following counterfactual question: 
\begin{quoting}
    \em What would have happened, had rifleman $A$ not shot? 
\end{quoting}

\paragraph{Interventional Counterfactual.}
The interventional interpretation of our counterfactual question is illustrated in~\cref{fig:firingsquad_interventional}.
Here, we use the factual observation $C=1$ together with~\eqref{eq:captain} for abduction and conclude that $U=1$.
Since the exogenous variables are assumed shared in the interventionist account, we conclude that the court still would have ordered the execution in the counterfactual world.
We then modify~\eqref{eq:riflemanA} to enforce the counterfactual antecedent $A^*=0$ through an intervention that locally violates the causal laws. 
This is indicated by the missing edge $C^*\to A^*$ in~\cref{fig:firingsquad_interventional}.
A possible interpretation is that rifleman $A$ did not shoot  because they disobeyed the order of the captain or their rifle got jammed.
Finally, we use the resulting modified SCM for prediction to conclude that $C^*=B^*=P^*=1$, meaning that the captain still would have signalled, rifleman $B$ still would have shot, and so the prisoner would still be dead.

\paragraph{Backtracking Counterfactual.}
The backtracking interpretation of our counterfactual question is illustrated in~\cref{fig:firingsquad_backtracking}.
Recall that in backtracking, we introduce new counterfactual exogenous variables $U^*$ which absorb differences across worlds, while the causal laws remain unchanged.  
Here, there are exactly two configurations of the variables that are compatible with the causal laws:
$U=C=A=B=P=1$ and $U=C=A=B=P=0$, that is, all variables are either one or zero.
From observing the counterfact $A^*=0$, we conclude that we must be in the latter case.
In other words, rifleman $A$ not firing could have only happened in the counterfactual world if the captain had not signalled ($C^*=0$) because the court did not order the execution ($U^*=0$). Hence, rifleman $B$ also would not have shot ($B^*=0$) and so the prisoner would still be alive ($P^*=0$).

\paragraph{Take-Away.}
\looseness-1 The point of the previous example is not to single out one interpretation as correct and the other as incorrect. Rather, our goal is to highlight the ambiguous nature of counterfactuals in an intuitive context: depending on the used semantics, the same counterfactual question can be answered in radically different ways (the prisoner is either dead or alive).
As discussed in~\cref{sec:related_work}, depending on the circumstances, background knowledge, and prior beliefs, counterfactuals can be interpreted differently. 
For example, a historian following the ``minimal rewrite rule'' 
might 
use their background knowledge and domain understanding to answer the question interventionally if rifleman $A$ not shooting because of disobeying the command constitutes a smaller perturbation to history than the alternative (backtracking) explanation that they perfectly obey orders and did not shoot because the captain did not signal. In the latter case, a similar decision would then need to be made as to whether the captain disobeyed the court's order (interventional) or whether the court did not order the execution (backtracking).
How a given counterfactual query is interpreted thus often depends on further background knowledge that is not made explicit in the causal model or counterfactual query. 
Interventional counterfactuals can be viewed as relying on an extreme form of such knowledge by always opting for a local violation of the laws (a small miracle), while leaving it implicit that they are, in fact, committing to this extreme.
We offer a formal representation that allows for making this knowledge explicit and for considering different choices, such as full or partial backtracking as alternatives.  
%
%
%

\end{document}